\documentclass{article}

     \usepackage[nonatbib,preprint]{neurips_2020}

\usepackage[utf8]{inputenc} % allow utf-8 input
\usepackage[T1]{fontenc}    % use 8-bit T1 fonts
\usepackage{hyperref}       % hyperlinks
\usepackage{url}            % simple URL typesetting
\usepackage{booktabs}       % professional-quality tables
\usepackage{amsfonts}       % blackboard math symbols
\usepackage{nicefrac}       % compact symbols for 1/2, etc.
\usepackage{microtype}      % microtypography

\usepackage{float}
\usepackage{ dsfont } % for fancy Reals R
\usepackage{graphicx} %image support
\usepackage{subcaption} %nested figures
\usepackage{csvsimple} %loading tables dynamically while drafting
\usepackage{amsthm} % for use of \begin{proof}
\newtheorem{theorem}{Theorem}[section] % for use of \begin{theorem}

\usepackage{amsmath} % aligned equations
\usepackage{comment} % for comment
\title{Training with Multi-Layer \\ Embeddings for Model Reduction}

\author{%
  Benjamin Ghaemmaghami $^1$, Zihao Deng $^1$, Benjamin Cho $^1$, Leo Orshansky $^2$, \\\textbf{Ashish Kumar Singh $^3$, Mattan Erez $^1$, and Michael Orshansky $^1$ }\\
  \\
  $^1$ Department of Electrical and Computer Engineering, University of Texas at Austin \\
  $^2$ Department of Computer Science, University of Texas at Austin \\
  $^3$ E2OPEN, India \\
}

\begin{document}
\maketitle

% note: 
% cannot use subfile package for latex ACM submission, but I think that it is nice for editing
% so we can keep the content separate from the latex setup

\begin{abstract}
Modern recommendation systems rely on real-valued embeddings of categorical features. Increasing the dimension of embedding vectors improves model accuracy but comes at a high cost to model size. We introduce a multi-layer embedding training (MLET) architecture that trains embeddings via a sequence of linear layers to derive superior embedding accuracy vs. model size trade-off.

Our approach is fundamentally based on the ability of factorized linear layers to produce superior embeddings to that of a single linear layer. We focus on the analysis and implementation of a two-layer scheme. Harnessing the recent results in dynamics of backpropagation in linear neural networks, we explain the ability to get superior multi-layer embeddings via their tendency to have lower effective rank. We show that  substantial advantages are obtained in the regime where the width of the hidden layer is much larger than that of the final embedding ($d$). Crucially, at conclusion of training, we convert the two-layer solution into a single-layer one: as a result, the inference-time model size scales as $d$. 

We prototype the MLET scheme within Facebook's PyTorch-based open-source Deep Learning Recommendation Model. We show that it allows reducing $d$ by 4-8X, with a corresponding improvement in memory footprint, at given model accuracy. The experiments are run on two publicly available click-through-rate prediction benchmarks (Criteo-Kaggle and Avazu). The runtime cost of MLET is 25\%, on average.

\end{abstract}

\section{Introduction}
Recommendation models (RMs) underlie a large number of applications and improving their performance is increasingly important. The click-through-rate (CTR) prediction task is a special case of general recommendation that seeks to predict the probability of a user clicking on a specific item, e.g. an ad, given the history of the user's past reactions. The user reactions and earlier-encountered instances are used in training the CTR model and are described by multiple features that capture user information (e.g., age, gender) and item information (e.g., movie title, cost) \cite{Ouyang2019ClickthroughRP}. Features are either numerical or categorical variables. 

A categorical variable with $n$ possible values can be represented by an $n$-dimensional one-hot vector. However, a fundamental aspect of modern recommendation models is their reliance on embeddings which map categorical variables into dense representations in an abstract real-valued space. Embeddings are superior for two main reasons. The first is that they allow a compacted representation compared to high-dimensional sparse one-hot, or multi-hot, direct encodings of categorical data. The second is that dense embedding vectors represent meaningful information that is exploited by RMs for improved performance: the angle (dot-product) between two embedding vectors represents their semantic similarity. Following a seminal innovation of Factorization Machines \cite{Rendle2010FactorizationM}, many modern RMs exploit this by using dot-products between embedding vectors to define the strength of feature interactions. 

State-of-the-art RMs increasingly rely on deep neural networks. Most high-performing models
use a combination of multi-layer perceptrons (MLPs) to process dense features, linear layers to generate embeddings of categorical features, and sub-networks that generate higher-order interactions. The outputs of the interaction sub-networks and MLPs are used as inputs into a linear (logistic) model with a sigmoid activation to produce the CTR prediction. 
Broadly, the above describes the architectures of Wide and Deep \cite{WideDeep}, Deep and Cross \cite{Wang2017DeepC}, DeepFM \cite{Guo2017DeepFMAF}, Field-Aware Factorization Machine (FFM) \cite{Juan2016FieldawareFM}, and xDeepFM \cite{Lian2018xDeepFMCE} networks, among others. The differences between the models are largely in how they handle the higher-order feature interactions. The Deep Learning Recommendation Model (DLRM) \cite{Naumov2019DeepLR}, that we use for prototyping our technique, is structurally similar to other models. DLRM does not include higher-order interactions, judging that their computational and memory cost is not justified. This is supported by empirical results on public datasets that show DLRM outperforming models with explicit higher-order interactions, such as the Deep and Cross model \cite{Wang2017DeepC}. 

All DNN-based RMs described above derive embeddings as part of model training through backpropagation. Algorithmically, embeddings are implemented as linear layers: if a categorical feature in one-hot encoding is a vector $q \in \mathds{Z}^{1\times n}$, then the embedding lookup is a vector-matrix multiplication $qW$. Here, $W \in \mathds{R}^{n\times d}$ is the embedding table (matrix) whose $i_{th}$ row represents the embedding of the $i_{th}$ category in a $d$-dimensional vector space. Conventionally, $W$ \emph{is implemented as a single linear layer and jointly trained  with the rest of the model to minimize the loss on the CTR task}.

Though embeddings are a more efficient representation of features compared to one-hot categorical vectors, the embedding tables still impose an increasingly heavy cost in system deployments, with tables commonly requiring tens of gigabytes of space \cite{Ginart2019MixedDE}. The reason is the large value of $n$: it is not uncommon to encounter a single categorical feature with millions of distinct values. For example, in the public Avazu dataset, one categorical feature has 6.7 million values. 

There are many techniques aiming to reduce the memory requirements of embedding tables - some unique to the embedding layer setting and others general. 
Compression-based techniques operate on trained layers and use pruning and quantization to reduce table size \cite{ling-etal-2016-word, nearlossless_tissier, sun2016sparse}. Low-rank approximation via SVD is another example of post-training compression \cite{bhavana2019block}. Other techniques perform pruning or quantization during training \cite{alvarez2017compression, naumov2018periodic}. 
While the above group of methods does not involve modifying the structure of the model, other methods, such as hashing and tensor factorization \cite{attenberg2009collaborative, khrulkov2019tensorized}, achieve superior quality-size behavior through a modified model structure that results in better use of model parameters. 
Using the unique properties of RMs, in \cite{ginart2019mixed}, a mixed-dimension strategy uses statistical patterns (frequency) of accessing individual entries to embed the popular entries into vectors of higher dimension compared to the less popular entries.

\begin{figure}[H]
	\centering
	\includegraphics[width=240pt]{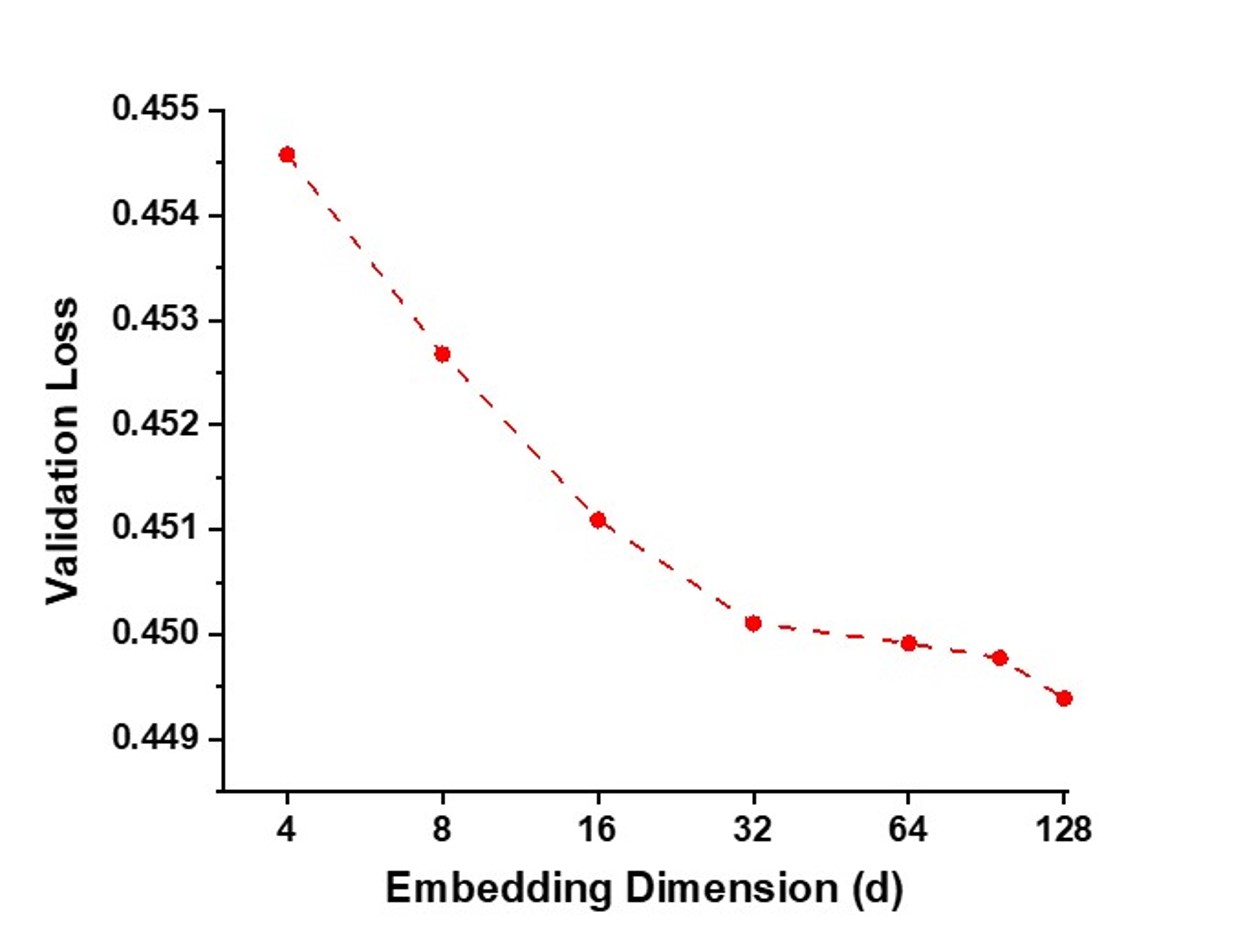}
	\caption{CTR model accuracy vs. the embedding dimension based on a single-layer embedding. The trade-off curve is generated using DLRM on the Criteo-Kaggle dataset.}
	\label{fig:tradeoff}
\end{figure}

Embedding vector dimension $d$ is a critical factor that controls the table size as well as model performance. Both empirical and theoretical evidence suggests that there exists a fundamental trade-off under which reducing vector dimension $d$ leads to the loss in model performance 
\cite{naumov2019dimensionality, Yin2018OnTD}. The trade-off is illustrated in Figure~\ref{fig:tradeoff} using DLRM on the Criteo-Kaggle dataset. 
\emph{The contribution of this paper is in developing a novel way of deriving a superior model size-accuracy trade-off.}

\subsection{Our Contribution: Multi-Layer Embedding Training}

We propose a novel way of achieving a smaller model size without accuracy degradation. The technique, which we call a multi-layer embedding training (MLET) architecture, trains embeddings via a sequence of linear layers, instead of a single layer. 

The fundamental underpinning for the superior behavior of MLET is the dynamics of training  using backpropagation. Harnessing recent results in the training of deep factorized linear neural networks, we provide a theoretical explanation for the surprising fact that multi-layer embeddings lead to a superior size-accuracy trade-off. 
We show that the main reason for the superior behavior is the impact of factorization on the generalization ability of the model, which is produced by the model's convergence towards a less complex solution.  

We focus on a prototype implementation that employs two linear layers. The inner dimension between the two layers is $k$. The second layer's output dimension is equal to the target embedding dimension $d$.
We find empirically that the effectiveness of the \emph{two-layer embedding technique}  depends heavily on the ratio $k/d$ for any given target embedding dimension $d$. The most benefit occurs when $k/d > 4$. The main cost of MLET is a $k/d$ increase in the required memory capacity during training (compared to a conventional embeddings training with dimension $d$).

It would appear that MLET increases the number of model parameters significantly, with the size of the embedding table increased by a factor of $k/d$. 
However, a two-layered approach is essential only during training. 
We eliminate the inference-time memory and model storage cost of MLET by a \emph{a post-training layer transformation} that collapses the multiple linear layers into a single one. As a result, for inference, only the original embedding table of size $n \times d$ is stored. 
    
We implement the proposed algorithmic framework in PyTorch using DLRM.
We demonstrate substantial benefits of MLET in terms of model size reduction of 4-8X, at constant accuracy on two public CTR datasets Avazu and Criteo-Kaggle. We find that the runtime cost of MLET is about 25\%.

\section{Superior Embedding Size-Accuracy Trade-off via Multi-Layer Embedding Training}
\subsection{Multi-Layer Embeddings: Definitions}

We now introduce the notation and details of MLET. Let the final embedding table be $W$ of size $n \times d$, where $n$ is the number of elements in the table and $d$ is the embedding dimension. 
\begin{equation}
    W \in \mathds{R}^{n \times d}
\end{equation}

We focus on a two-layer architecture and seek to factorize the embedding table $W$ in terms of  $W_1$ and $W_2$:
\begin{equation}
    W = W_1  W_2
\end{equation}
\begin{equation}
    W_1 \in \mathds{R}^{n \times k}
\end{equation}
\begin{equation}
    W_2 \in \mathds{R}^{k \times d}
\end{equation}

Let the row vector $q\in \mathds{Z}^{1\times n}$ denote a one-hot encoding of a feature with $n$ categories. The embedding lookup is represented by a vector-matrix product: 

\begin{equation}
    r = qW_1W_2
\end{equation}

Here, $r\in \mathds{R}^{1\times d}$ is the embedding of $q$ in a $d$-dimensional space. 
$W_1$ and $W_2$ are trained jointly. After training there is no need to keep both $W_1$ and $W_2$, and we only store their product, $W = W_1W_2$. This reduces a two-layer embedding into a single one for inference-time evaluation and storage.

The essential aspect of MLET's training of an embedding using a sequence of two linear layers are the relative dimensions of $W_1$ and $W_2$. 
As defined above, $W_1$ and $W_2$ are of shape $n\times k$ and $k\times d$, respectively.
We say that a model with a linear layer $n\times d_1$ dominates (>>) another linear model with a linear layer $n\times d_2$ if the validation loss on the first model is lower than that of the second model. Symbolically, ($n\times d_1$) >> ($n\times d_2$) if Loss ($n\times d_1$) < Loss ($n\times d_2$). 

For a single-layer model, the accuracy-size trade-off discussed earlier, and shown in Figure~\ref{fig:tradeoff}, can be restated as follows: ($n\times d_1$) >> ($n\times d_2$) if $d_1 > d_2$. Similarly, as a consequence of the same trade-off, it seems self-evident that for a two-layer linear model, the following holds: $(n\times k)\times (k \times d_1) $ >> $(n\times k)\times (k \times d_2) $ if $d_1 > d_2$ (this is also confirmed empirically, and can be seen in Figures 2 and 3).

Yet there are two aspects of MLET that seem quite surprising. 
The first is why a two-layer embedding is superior to a one-layer embedding, or compactly, why $(n\times k)\times (k \times d) $ >> $(n\times d)$? The second is why a two-layer model improves with a larger width of the hidden layer $k$, or compactly, why $(n\times k_1)\times (k_1 \times d) $ >> $(n\times k_2)\times (k_2 \times d) $ if $k_1 > k_2$? In the next section we explain the first behavior - the effect of factorization per se. We currently attribute the second behavior to the general tendency of overparameterized linear neural networks to positively depend on the width of hidden layers. We plan to explore this aspect of MLET more thoroughly in our future work.

\subsection{Why Does Factorization Help?}
Why should we expect to get a better embedding if we factorize the linear layer? Specifically, as we demonstrate in Section 2.2.2, in the MLET operating regime of $k \geq d$, any embedding defined by a two-layer model lies in the search space of a single-layer model. 
Therefore, if there is an optimal solution found by a two-layer model, our intuition is that a single-layer model should also be able to find it, and, thus, a two-layer model should not be better than a single-layer model. Yet, empirically, we find that two-layer models consistently outperform their single-layer counterparts. 

To understand why factorization helps, we rely on recent results in the dynamics of training linear layers using backpropagation. \textit{The main reason for the superior behavior of the multi-layer model training is the impact of factorization on the generalization ability of the model.} It achieves this by convergence towards a less complex solution.

\subsubsection{Dynamics of Factorized Linear Layer Network Training}
 It is a widely accepted notion in deep learning that low-rank weight matrices lead to better generalization and help avoid overfitting \cite{Arora2018StrongerGB}.  
 Practically, regularization on rank is a common and powerful approach to restrict model complexity and thus enhance generalization. A variety of ML algorithms use regularization on rank to achieve better generalization, including robust principal component analysis \cite{Gu2016WeightedNN}\cite{Sun2013RobustPC}, robust matrix completion  \cite{Chen2011MatrixCW}, subspace clustering \cite{Peng2015SubspaceCU}\cite{Liu2010RobustSS}, and others \cite{Gu2014WeightedNN}. %In the space of recommendation models, the Factorization Machine \cite{Rendle2010FactorizationM} also gains its generalization power by explicitly constraining the rank of the second-order feature interaction matrix.

Recent work \cite{Arora2019ImplicitRI} has shown that a linear layer network (LLN) with multiple layers has a strong bias towards learning a low-rank weight matrix. 
\emph{The fundamental reason behind this is that the process of training an LLN with the gradient descent algorithm results in larger polarization of the singular values of the learned matrix in LLNs with more layers}. The result is that large singular values are amplified while small ones are attenuated and tend to vanish.

Consider an $N$-layer LLN with a weight matrix of each layer being $W_{i}$. 
Let $W$ be the weight matrix that represents the LLN in a single layer form, i.e., $W=W_{1} \times W_{2}... \times W_{N}$. 
Let $\sigma_r$ denote the $r_{th}$ singular value of $W$. Let $u_r$ and $v_r$ be the $r_{th}$ left and right singular vectors of $W$, respectively. Let the loss function be $L$ and $\nabla L(W(t))$ be its gradient with respect to $W$ at time $t$. Given a learning rate $\eta$, the updates of the singular values are given by Eq. \ref{eq:svmovement} (Theorem 3 in \cite{Arora2019ImplicitRI}):
\begin{equation}
    \sigma_{r}(t+1) \leftarrow \sigma_{r}(t)-\eta \cdot N \cdot\left(\sigma_{r}(t)\right)^{2-2 / N}\cdot\left\langle\nabla L(W(t)), \mathbf{u}_{r}(t)         \mathbf{v}_{r}^{\top}(t)\right\rangle
    \label{eq:svmovement}
\end{equation}
Critically, the term $\left(\sigma_{r}(t)\right)^{2-2 / N}$ captures the dependence on the number of layers. For a single-layer model ($N=1$), the term reduces to 1 for all $r$, making the update to $\sigma_{r}(t)$ independent of its current value. However, for a multi-layer LLN, the update term grows, linearly or faster, with the current value of $\sigma_{r}(t)$. 
For $\sigma_{r}(t)<1$ , the term $\left(\sigma_{r}(t)\right)^{2-2 / N}$ is strictly less than 1 and gets smaller for smaller $\sigma_{r}(t)$. Therefore, a multi-layer LLN attenuates the updates for small singular values. By the same reasoning, a multi-layer LLN enhances the updates for large singular values. As $N$ increases, the gap between larger and smaller singular values increases, resulting in  $W$ having lower rank.  

We directly observed the bias towards the low-rank weight matrices by analyzing the distribution of singular values of the embedding matrices in our MLET experiments. The Avazu dataset has 21 categorical features but two of them have far more items than the rest: feature-9 and feature-10 are jointly responsible for $99.7\%$ of all embedding table entries. 
Now consider 8 singular values of embeddings learned using a single-layer model with $d=8$ and those from the MLET model with $k=64$ and $d=8$. For feature-9, all 8 singular values of the single-layer model are larger than $0.01$ of its largest singular value. However, only 2 singular values of the embedding produced by MLET are larger than $0.01$ of the largest one. Similarly, for feature-10, 5 singular values of the single-layer model are larger than $0.01$ of its largest singular value but only 2 singular values of the MLET model are larger than $0.01$ of the largest one.

We use the above tendency of factorized linear layers to produce a lower-rank $W$ to derive superior category embeddings by replacing a single-layer embedding with a multi-layer embedding. In the experiments we describe in Section 4, we find that using $N=2$ is sufficient and using higher $N$ is not helpful. 

\subsubsection{Dimensional Constraints in Multi-Layer Embeddings} 
\label{section-k}
In MLET, the dimensions of the layers are important. For concreteness, we focus on a two-layer embeddings ($N=2$) and explain why MLET requires $k \geq d$. The theory of rank regularization does not make any assumptions about the relation between $k$ and $d$. 
The reason for imposing the constraint that $k \geq d$ is to ensure that the search space of a two-layer model and that of its single-layer counterpart are identical. With this condition satisfied, the tendency of a multi-layer model towards a low-rank solution leads to superior generalization. When $k<d$, however, the search space of the multi-layer model is reduced to a 
subset of a single-layer model's search space. This counteracts the benefits of a lower-rank solution with no guaranteed improvement in generalization. We do not propose to operate in this regime.  

Again, let the two matrices in the two-layer model be $W_1 \in \mathds{R}^{n \times k}$ and $W_2\in \mathds{R}^{k \times d}$. 
Let the matrix in the single-layer model be $W \in \mathds{R}^{n\times d}$. 
The search space of a single-layer model is the set of linear transformations defined by all possible matrices $W$: $\{W|W\in\mathds{R}^{n \times d}\}$. 
The search space of a two-layer model is then the set of linear transformations defined by all possible products of $W_1W_2$: $\{W|W=W_1W_2,W_1\in\mathds{R}^{n \times k},W_2\in\mathds{R}^{k \times d}\}$.

First, we formally prove that for $k\geq d$, the search space of a two-layer model is the same as that of a single-layer model. Then, we prove that for $k < d$, the search space of a two-layer model is reduced to a subset of the search space of a single-layer model.

\begin{theorem} The search space of a two-layer model $W_1W_2$ is the same as that of a single-layer model $W$ when $k\geq d$: $\{W|W\in\mathds{R}^{n \times d}\}=\{W|W=W_1W_2,W_1\in\mathds{R}^{n \times k},W_2\in\mathds{R}^{k \times d}\}$.  
\end{theorem}

\begin{proof}
We first show that $\{W|W\in\mathds{R}^{n \times d}\} \subseteq \{W|W=W_1W_2,W_1\in\mathds{R}^{n \times k},W_2\in\mathds{R}^{k \times d}\}$ by showing that $\text{for }\forall W\in\mathds{R}^{n \times d}\text{, } \exists W_1\in\mathds{R}^{n \times k}\text{ and }W_2\in\mathds{R}^{k \times d} \text{, such that } W=W_1W_2$. 

Consider the QR decomposition of $W^T$:
\begin{equation}
\begin{split}
    W^T = QR \\
    Q \in\mathds{R}^{d\times d} \\
    R \in\mathds{R}^{d\times n}
\end{split}
\label{eq:QR}
\end{equation}
Let $I_{dk}$ denote the matrix constructed by concatenating ($k-d$) zero columns to a $d\times d$ identity matrix:
\begin{equation}
\begin{split}
    W_1 &= R^TI_{dk} \\
    W_2 &= I_{dk}^{T}Q^T
\end{split}
\label{eq:W1W2}
\end{equation}
It follows that $W_1\in\mathds{R}^{n \times k},W_2\in\mathds{R}^{k \times d}$ and $W=W_1W_2$. 

We now show that $\{W|W=W_1W_2,W_1\in\mathds{R}^{n \times k},W_2\in\mathds{R}^{k \times d}\} \subseteq \{W|W\in\mathds{R}^{n \times d}\}$. This follows from the fact that any matrix represented by a product of $W_1W_2$ can be represented by a single matrix $W$ by simply letting $W=W_1W_2$. 
\end{proof}

\begin{theorem}
The search space of a two-layer model $W_1W_2$ is reduced compared to that of a single-layer model $W$ when $k < d$:
$\{W|W\in\mathds{R}^{n \times d}\} \supset \{W|W=W_1W_2,W_1\in\mathds{R}^{n \times k},W_2\in\mathds{R}^{k \times d}\}$.
\end{theorem}

\begin{proof}
Let $W$ be such that $\mathrm{Rank}(W)>k$. Clearly, for both $W_1$ and $W_2$, $\mathrm{Rank}(W_1) \leq k$ and $\mathrm{Rank}(W_2) \leq k$. Further, $\mathrm{Rank}(W_1W_2)\leq \min(\mathrm{Rank}(W_1),\mathrm{Rank}(W_2)) = k$. Thus, there does not exist $W_1$, $W_2$ for which $W=W_1W_2$.

\end{proof}

\section{Experiments}
\label{section-experiment}
We evaluate the proposed algorithm on two public datasets for click-through rate tasks: Criteo-Kaggle and Avazu. Both datasets are composed of a mix of categorical and real-valued features (Table~\ref{tab:dataset_info}). Both datasets are split into training, testing, and validation sets of 80\%, 10\%, and 10\%, respectively. The Criteo-Kaggle dataset was split based on the time of data collection: the first six days are used for training and the seventh day is split evenly into the test and validation sets. The Avazu dataset was split randomly. The models are implemented in PyTorch. 
The experiments are run on two systems: (a) an Intel i7-9700 CPU hosting an NVIDIA RTX2080 GPU with 8GB GDDR, and (b) an Intel i7-8700 CPU hosting an NVIDIA Titan Xp GPU with 12GB GDDR. All experiments were run on the GPUs 
except for the Criteo-Kaggle dataset with $k \geq 128$. In that case, the required memory exceeded 12GB and the experiments were performed on the CPU of system (b).
CPU throughput is 2 to 3 times lower compared to that of a GPU depending on the configuration. For consistency, all runtime estimates are produced from experiments run on NVIDIA RTX2080. 

\begin{table}[ht]
\centering
\caption{Dataset Composition}
\begin{tabular}{@{}llll@{}}
\toprule
Dataset       & Total Records & Dense Features & Categorical Features \\ \midrule
Criteo-Kaggle & 45,840,617    & 13             & 26                   \\
Avazu         & 40,400,000    & 1              & 21                   \\ \bottomrule
\end{tabular}
\label{tab:dataset_info}
\end{table}

DLRM has several hyperparameters. For both datasets we configure DLRM’s top MLP to have two hidden layers with 512 and 256 nodes. For the Avazu dataset, we set DLRM’s bottom MLP to be $256 \rightarrow 128 \rightarrow d$. For the Criteo-Kaggle dataset, we configure DLRM’s bottom MLP to be $512 \rightarrow 256 \rightarrow 128 \rightarrow d$. The bottom MLPs differ because their role is to handle the real-valued features which vary between datasets. In all experiments, $d$ is set equal to the embedding dimension so that vector sizes for the real-valued and categorical features match.

Following prior work  \cite{Naumov2019DeepLR}, we train the models only for a single epoch, with a universal learning rate of 0.2 and a batch of 128  using SGD as the optimizer. The linear factorization layers are initialized using a Gaussian distribution $ \sim N(0,0.0625)$. The initialization is unique for each embedding table.

For each hyperparameter configuration, at least five training runs are performed to decrease the impact of initialization variation and run-to-run variation due to non-deterministic GPU execution. 
We find that the initial state of the embedding tables has a non-negligible impact on overall model performance after training. Additionally, even with the same initial conditions, we observe run-to-run variations in the resulting model performance when using a GPU. We ascribe such run-to-run variation to the documented non-determinism of the CUDA implementation of some PyTorch operators, such as EmbeddingBag \cite{PytorchNondeterministicCUDA}. The reported data is based on the mean values of the replicated runs. We report two performance metrics: area under the ROC curve (AUC) and binary cross-entropy (LogLoss).
Recall that $d$ defines the size of the inference-time embedding vectors (and, the table) while $k$ refers to the width of the hidden linear layer in MLET. 

The experiments demonstrate the effectiveness of MLET in producing superior model size vs. performance trade-offs compared to the baseline single-layer embeddings implementation using DLRM. Figures~\ref{fig:edim_vs_auc}--\ref{fig:edim_vs_logloss} summarize the results. 

As originally intended, the predominant system benefit made possible by the superior model size vs. performance curves is in terms of reducing the embedding table size. As Figures~\ref{fig:edim_vs_auc}--\ref{fig:edim_vs_logloss} show, at the same model accuracy levels, we are able to produce models with $d$ (and, therefore the embedding table size) 4-8 times smaller compared to the baseline DLRM. 

The benefits begin to be observed in MLET curves even for $k=d$. Increasing $k$ for a given $d$ leads to a monotonic improvement in model accuracy. 
For CTR systems, an improvement of  0.001 in LogLoss is considered substantial. The maximum LogLoss benefit of MLET for Criteo-Kaggle is 0.0025, and the maximum benefit for Avazu is 0.006. This improvement in model accuracy saturates as $k$ grows, e.g., for the Criteo-Kaggle dataset the curves with $k=64$ and $k=128$ are very similar. 

We further observe that the relative performance improvements are largely defined by $k/d$. Recall that MLET results in training memory increase of $k/d$ compared to a single-layer training algorithm. Since in practice we need to operate under a certain memory budget, there is a limit to the achievable $k/d$. Specifically, the higher values of $k/d$ can be achieved for smaller $d$, thus our technique is most effective at lower $d$.

So far, we analyzed  either the table size reduction at a given accuracy, or accuracy improvement at a given table size, with training memory requirement being the cost (in the sense that achieving both benefits requires $k>d$). Interestingly, we find there are also $k$ and $d$ combinations in which \emph{both} accuracy and size can be improved at, effectively, zero cost in terms of larger training-time memory. We say that a solution has zero cost if $k \leq d$. Symbolically, we can write that an MLET solution $(k,d)$ dominates (>>) a single-layer training solution $(d)$ if both accuracy and table size are improved. In Fig. 2b, we see that several points exhibit such behavior in the Avazu dataset experiments: (64,64) >> 128, (64,32) >>  128, (64,16) >> 128, (32,16) >> 32, and (32,8) >> 32. Such behavior appears dataset-dependent since we do not find such cost-free solutions for the Criteo-Kaggle dataset.

\begin{figure}[H]
	\centering
	\begin{subfigure}{.45\textwidth}
		\includegraphics[width=\textwidth]{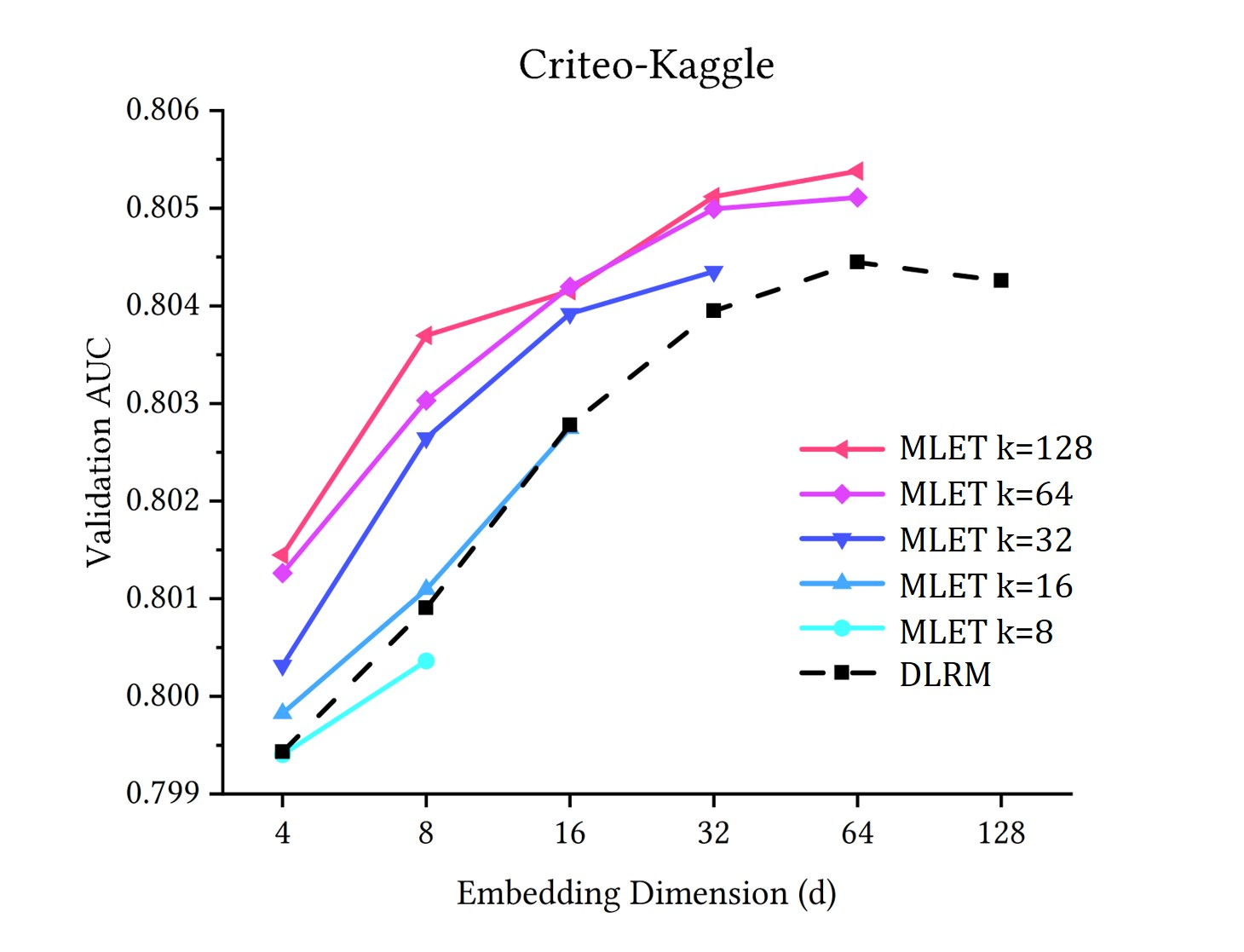}
		\caption{}
	\end{subfigure}
%%%%%%%%%%%%%%
	\begin{subfigure}{.45\textwidth}
		\includegraphics[width=\textwidth]{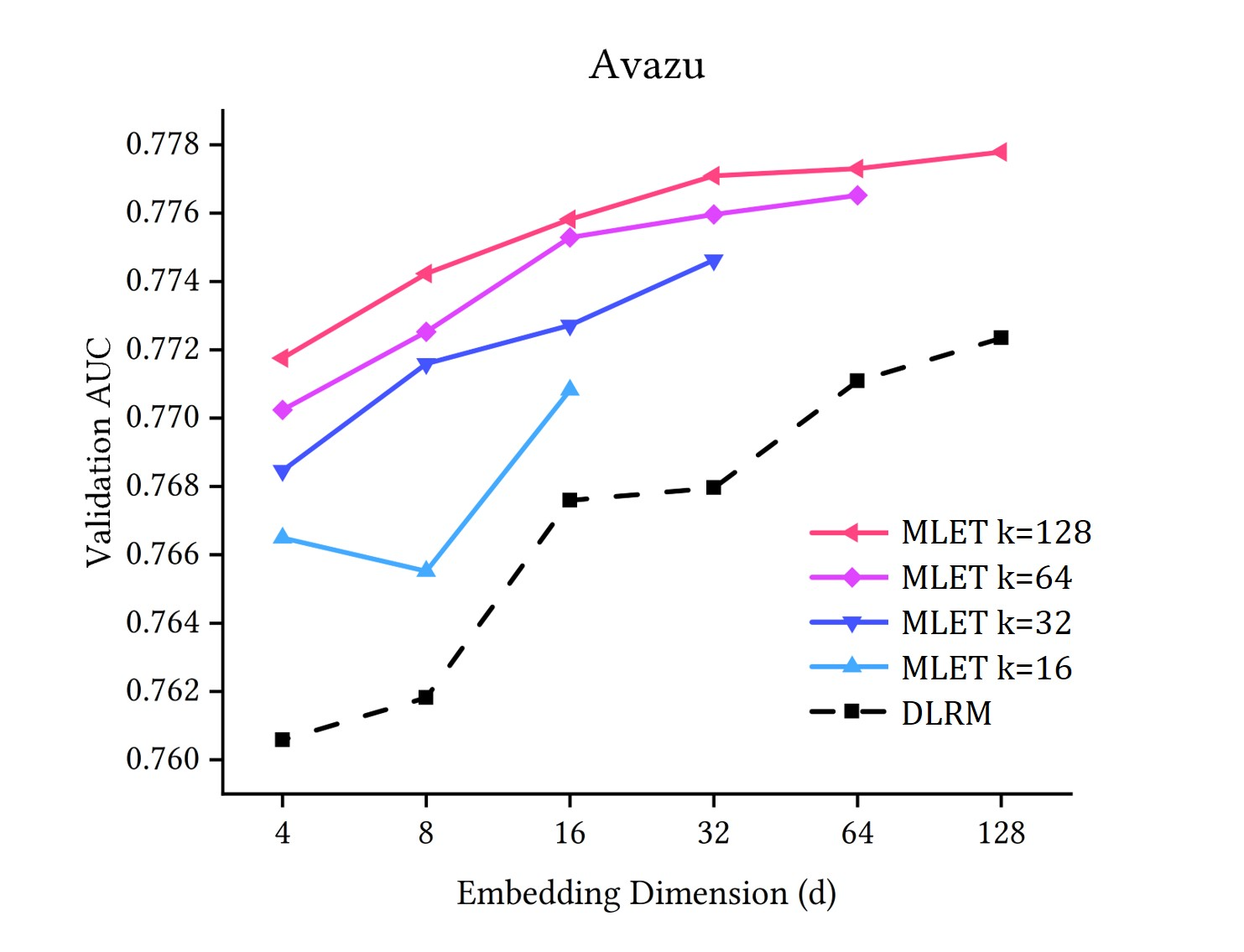}
		\caption{}
	\end{subfigure}
	
	\caption{Validation Area Under the Curve (AUC) for Criteo-Kaggle and Avazu datasets.}
	\label{fig:edim_vs_auc}
%%%%%%%%%%%%%%
\end{figure}

\begin{figure}[H]
	\centering
	\begin{subfigure}{.45\textwidth}
		\includegraphics[width=\textwidth]{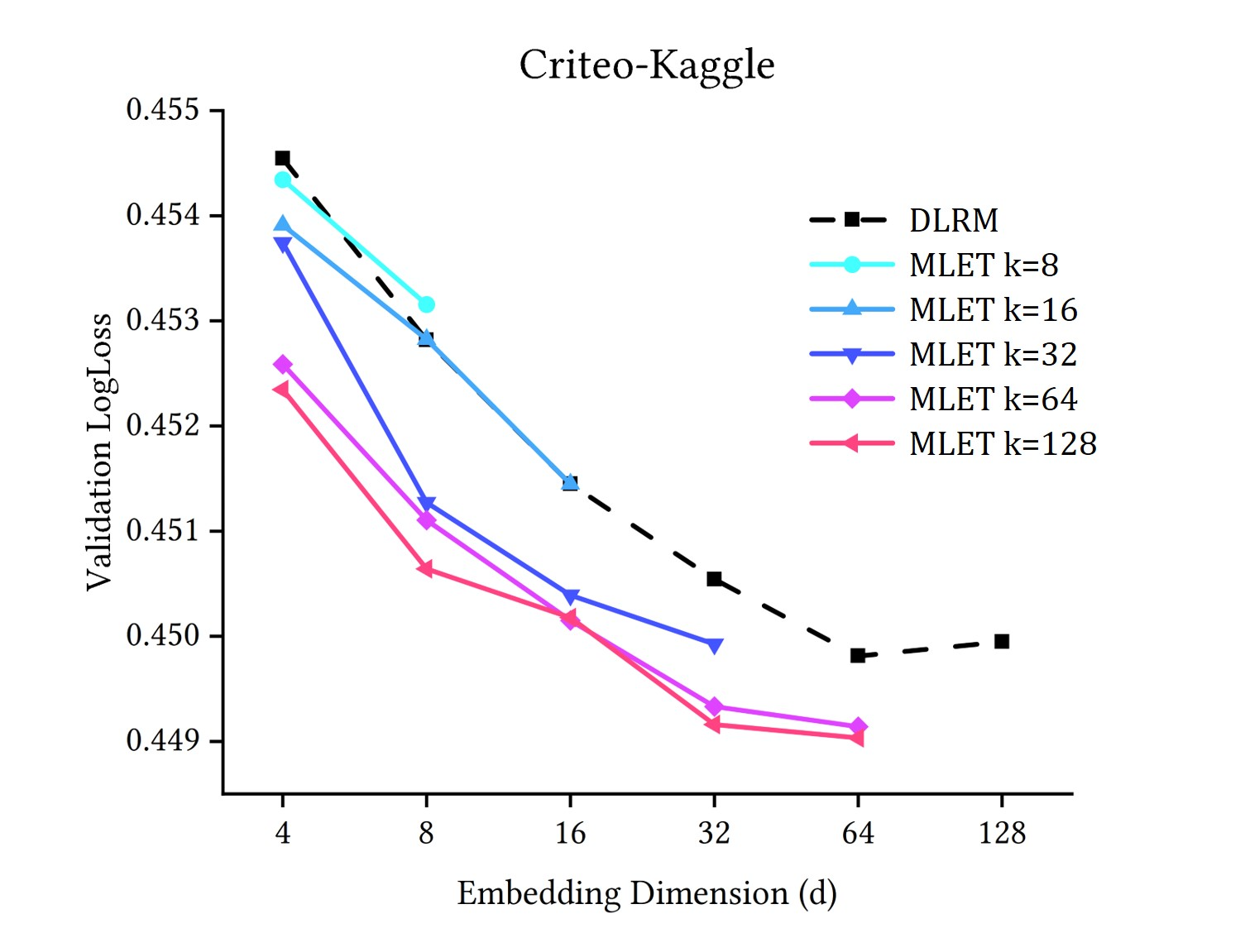}
		\caption{}
	\end{subfigure}
%%%%%%%%%%%%%%
	\begin{subfigure}{.45\textwidth}
		\includegraphics[width=\textwidth]{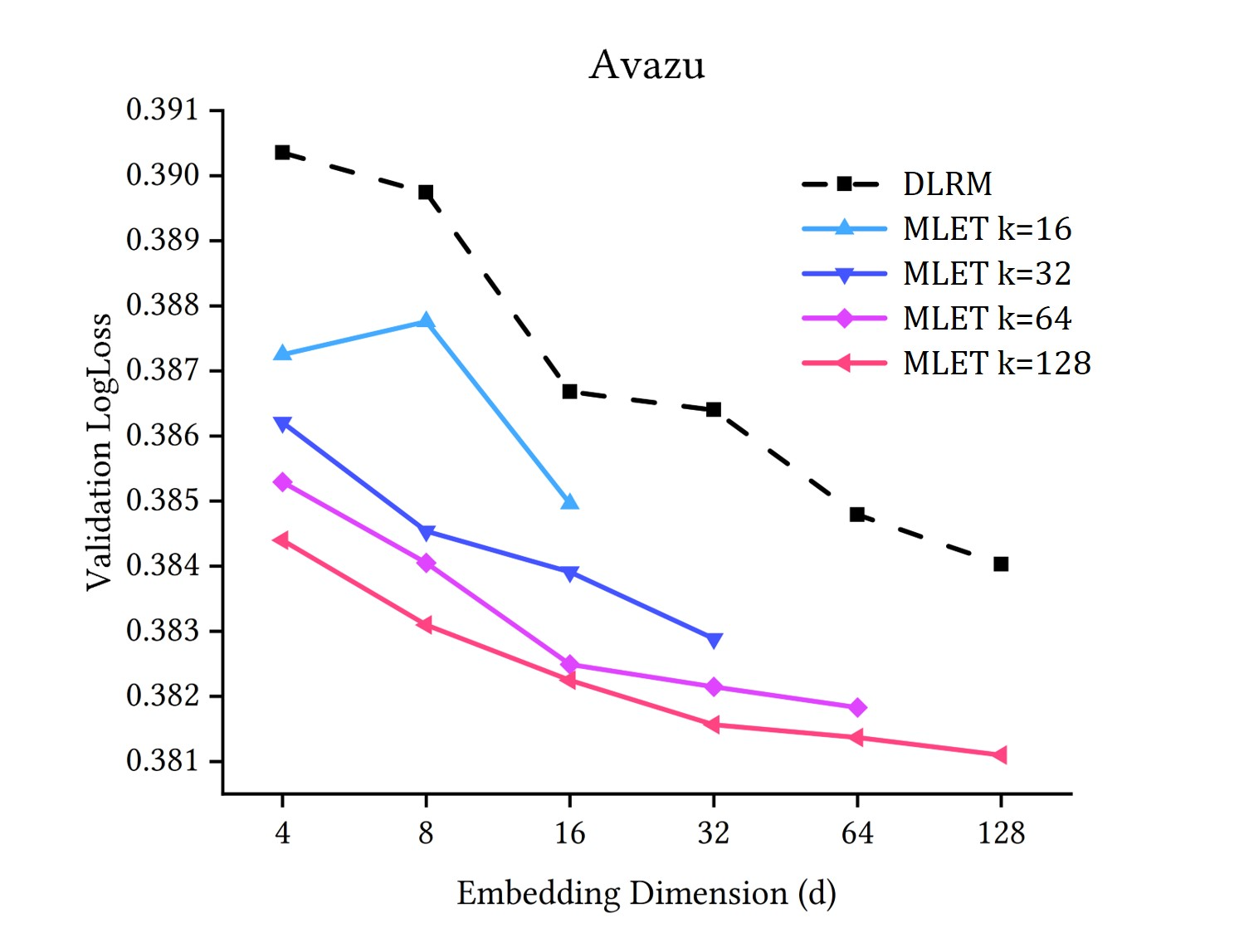}
		\caption{}
	\end{subfigure}
	
	\caption{Validation LogLoss for Criteo-Kaggle and Avazu datasets.}
	\label{fig:edim_vs_logloss}
%%%%%%%%%%%%%%
\end{figure}

The primary system impact of MLET is on training memory with some impact on runtime. MLET results in training memory increase of $k/d$ compared to a single-layer training algorithm (with embedding dimension $d$). We observe that for most experiments the memory requirement was below 12GB and exceeded that only for the Criteo-Kaggle dataset with $k \geq 128$.  
At inference time, MLET memory consumption is equivalent to a single-layer DLRM model with embedding dimension $d$.
In our naive implementation, the runtime cost of MLET training is 25\% compared to DLRM. The runtimes in terms of time per training iteration for various $k$, $d$ pairs on NVIDIA RTX 2080 are summarized in Table~\ref{table:k_d_perf}.

\begin{table}[]
\centering
\caption{DLRM and MLET training runtime on RTX2080}

\begin{tabular}{@{}lllr@{}}
\toprule
MODEL & $k$  & $d$  & ms/iteration \\ \midrule
DLRM  & 4  & 4  & 6.565              \\
DLRM  & 8  & 8  & 6.548              \\
DLRM  & 16 & 16 & 6.523              \\
DLRM  & 32 & 32 & 6.530              \\
DLRM  & 64 & 64 & 6.545              \\
MLET   & 8  & 4  & 8.140              \\
MLET   & 16 & 4  & 8.141              \\
MLET   & 32 & 4  & 8.314              \\
MLET   & 64 & 4  & 8.167              \\
MLET   & 8  & 8  & 8.118              \\
MLET   & 16 & 8  & 8.083              \\
MLET   & 32 & 8  & 8.175              \\
MLET   & 64 & 8  & 8.156              \\
MLET   & 16 & 16 & 8.088              \\
MLET   & 32 & 16 & 8.160              \\
MLET   & 64 & 16 & 8.230              \\
MLET   & 32 & 32 & 8.106              \\
MLET   & 64 & 32 & 8.124              \\
MLET   & 64 & 64 & 8.120              \\ \bottomrule
\end{tabular}
\label{table:k_d_perf}
\end{table}

\section{Conclusion}
In this paper, we introduced a multi-layer embedding training  architecture that trains embeddings via a sequence of linear layers to derive a superior embedding accuracy vs. model size trade-off. We provide an explanation for obtaining superior embeddings based on the theory of dynamics of backpropagation in linear layer neural networks. We prototyped the MLET scheme within Facebook's PyTorch-based open-source Deep Learning Recommendation Model and demonstrated that it allows reducing memory footprint by 4-8X without model accuracy degradation. 

\section {Acknowledgements}
We gratefully acknowledge the generous support of Facebook Research under the "AI System Hardware/Software Co-Design" program. We thank Maxim Naumov, Dheevatsa Mudigere, Constantine Caramanis, and Sujay Sanghavi for many helpful discussions.

\bibliographystyle{acm}
\bibliography{main.bib}

\end{document}